\documentclass{article}

\usepackage{microtype}
\usepackage{graphicx}
\usepackage{subfigure}
\usepackage{booktabs} %

\usepackage{hyperref}

\usepackage[accepted]{icml2024}

\usepackage{amsmath}
\usepackage{amssymb}
\usepackage{mathtools}
\usepackage{amsthm}

\usepackage[capitalize,noabbrev,nameinlink]{cleveref}

\theoremstyle{plain}
\newtheorem{theorem}{Theorem}[section]
\newtheorem{proposition}[theorem]{Proposition}
\newtheorem{lemma}[theorem]{Lemma}

\newtheorem{conjecture}[theorem]{Conjecture}
\theoremstyle{definition}
\newtheorem{definition}[theorem]{Definition}

\theoremstyle{remark}

\icmltitlerunning{How Private are DP-SGD implementations?}

\renewcommand{\texttt}[1]{{\fontfamily{cmtt}\selectfont #1}}

\usepackage{bm}
\usepackage{hyperref}
\hypersetup{
	colorlinks=true,
	linkcolor=blue,
	citecolor=blue, %
	filecolor=blue, %
	urlcolor=blue %
}

\usepackage{xspace}
\usepackage{enumitem}

\newcommand{\PARAMETERS}{\STATE \hspace{-3.5mm}{\bf Parameters:}\xspace}

\def\ddefloop#1{\ifx\ddefloop#1\else\ddef{#1}\expandafter\ddefloop\fi}
\def\ddef#1{\expandafter\def\csname #1\endcsname{\ensuremath{\mathbb{#1}}}}
\ddefloop ABCDFGHIJKLMNOQRSTUVWXYZ\ddefloop  %
\def\ddef#1{\expandafter\def\csname c#1\endcsname{\ensuremath{\mathcal{#1}}}}
\ddefloop ABCDEFGHIJKLMNOPQRSTUVWXYZ\ddefloop  %
\def\ddef#1{\expandafter\def\csname b#1\endcsname{\ensuremath{\bm #1}}}
\ddefloop ABCDEFGHIJKLMNOPQRSTUVWXYZabcdeghijklnopqrstuvwxyz\ddefloop  %

\newcommand{\eps}{\varepsilon}

\newcommand{\prn}[1]{\left ( #1 \right )}

\newcommand{\set}[1]{\left\{ #1 \right\}}

\newcommand{\Ncdf}{\Phi} %

\newcommand{\DP}{\mathsf{DP}}

\newcommand{\DPSGD}{\mathsf{DP}\text{-}\mathsf{SGD}}

\newcommand{\ABLQB}{\textsc{ABLQ}_{\cB}}
\newcommand{\ABLQD}{\textsc{ABLQ}_{\cD}}
\newcommand{\ABLQP}{\textsc{ABLQ}_{\cP}}
\newcommand{\ABLQS}{\textsc{ABLQ}_{\cS}}
\newcommand{\epsB}{\eps_{\cB}}
\newcommand{\epsD}{\eps_{\cD}}
\newcommand{\epsP}{\eps_{\cP}}
\newcommand{\epsS}{\eps_{\cS}}
\newcommand{\deltaB}{\delta_{\cB}}
\newcommand{\deltaD}{\delta_{\cD}}
\newcommand{\deltaP}{\delta_{\cP}}
\newcommand{\deltaS}{\delta_{\cS}}

\newcommand{\PD}{P_{\cD}}
\newcommand{\QD}{Q_{\cD}}
\newcommand{\PS}{P_{\cS}}
\newcommand{\QS}{Q_{\cS}}
\newcommand{\PP}{P_{\cP}}
\newcommand{\QP}{Q_{\cP}}

\begin{document}

\twocolumn[
\icmltitle{How Private are DP-SGD Implementations?}

\icmlsetsymbol{equal}{*}

\begin{icmlauthorlist}
\icmlauthor{Lynn Chua}{goog}
\icmlauthor{Badih Ghazi}{goog}
\icmlauthor{Pritish Kamath}{goog}
\icmlauthor{Ravi Kumar}{goog}
\icmlauthor{Pasin Manurangsi}{goog}
\icmlauthor{Amer Sinha}{goog}
\icmlauthor{Chiyuan Zhang}{goog}
\end{icmlauthorlist}

\icmlaffiliation{goog}{Google Research}

\icmlcorrespondingauthor{Pritish Kamath}{pritish@alum.mit.edu}
\icmlcorrespondingauthor{Pasin Manurangsi}{pasin@google.com}

\icmlkeywords{Machine Learning, ICML}

\vskip 0.3in
]

\printAffiliationsAndNotice{}  %

\begin{abstract}

We demonstrate a substantial gap between the privacy guarantees of the {\bf A}daptive {\bf B}atch {\bf L}inear {\bf Q}ueries (ABLQ) mechanism under different types of batch sampling: (i) Shuffling, and (ii) Poisson subsampling; the typical analysis of {\bf D}ifferentially {\bf P}rivate {\bf S}tochastic {\bf G}radient {\bf D}escent (DP-SGD) follows by interpreting it as a post-processing of ABLQ.
While shuffling-based DP-SGD is more commonly used in practical implementations, it has not been amenable to easy privacy analysis, either analytically or even numerically.
On the other hand, Poisson subsampling-based DP-SGD is challenging to scalably implement,
but has a well-understood privacy analysis, with multiple open-source numerically tight privacy accountants available.
This has led to a common practice of using shuffling-based DP-SGD in practice, but using the privacy analysis for the corresponding Poisson subsampling version.
Our result shows that there can be a substantial gap between the privacy analysis when using the two types of batch sampling, and thus advises caution in reporting privacy parameters for DP-SGD.
\end{abstract}

\section{Introduction}\label{sec:intro}

Using noisy gradients in first-order methods such as stochastic gradient descent (SGD) has become a prominent approach for adding differential privacy ($\DP$) to the training of differentiable models such as neural networks. This approach, introduced by \citet{abadi16deep}, has come to be known as Differentially Private Stochastic Gradient Descent, and we use the term $\DPSGD$ to refer to any such first-order method. $\DPSGD$ is currently the canonical algorithm for training deep neural networks with privacy guarantees, and there currently exist multiple open source implementations, such as in \citet{tf_privacy},
Pytorch Opacus~\citep{yousefpour21opacus},
and JAX Privacy~\cite{jax-privacy2022github}.
The algorithm has been applied widely in various machine learning domains, such as training of
image classification~\cite{tramer2020differentially,papernot2021tempered,klause2022differentially,de22unlocking,bu2022scalable},
generative models with GAN~\citep{Torkzadehmahani_2019_CVPR_Workshops,NEURIPS2020_9547ad6b},
diffusion models~\citep{dockhorn2022differentially},
language models~\cite{li2021large,yu2021differentially,anil22dpbert,he2022exploring},
medical imaging~\citep{ziller2021medical}, as well as private spatial querying~\citep{zeighami2021neural}, ad modeling~\cite{denison23ad}, and recommendation~\citep{fang2022differentially}.

\begin{figure}
\centering
\includegraphics[width=8cm]{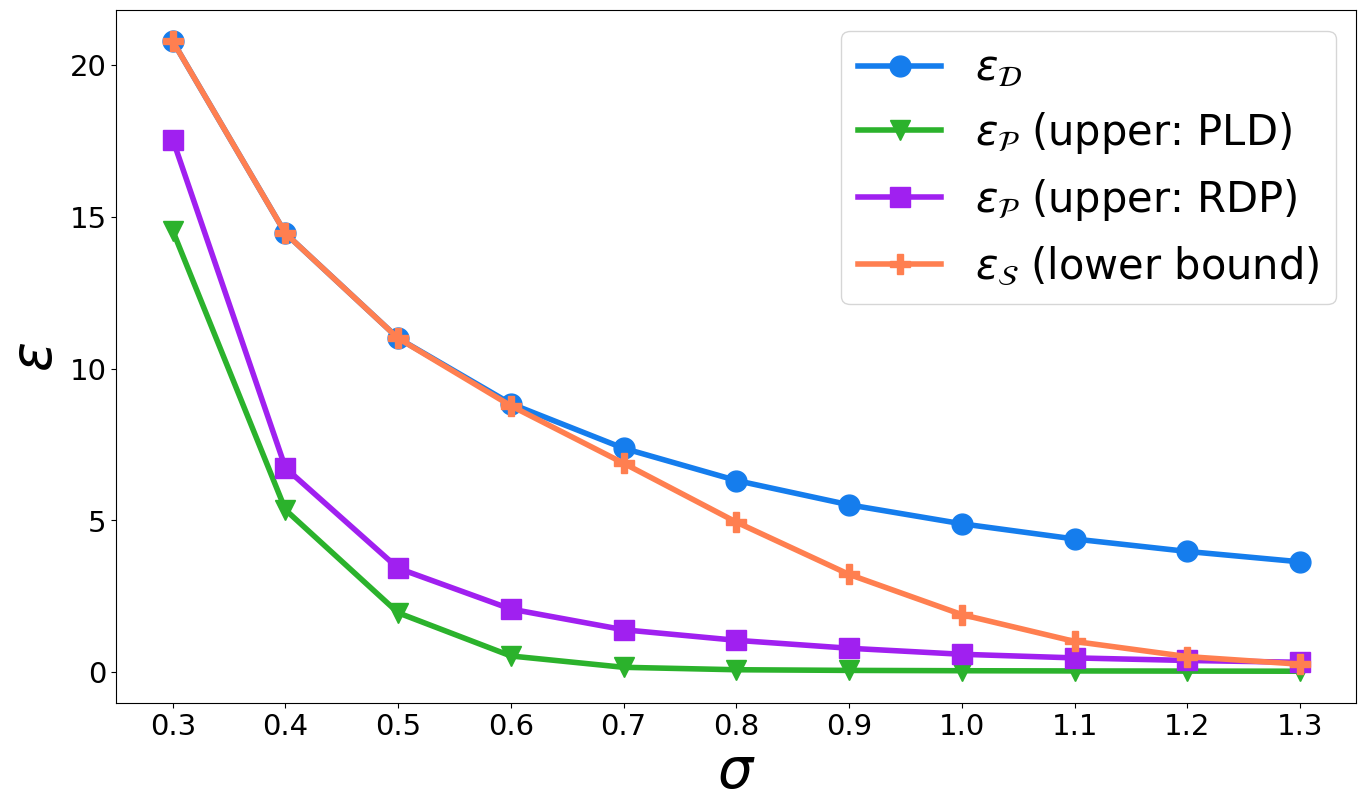}
\caption{Privacy parameter $\eps$ for different noise parameters $\sigma$, for fixed $\delta = 10^{-6}$ and number of steps $T = 10,000$. $\epsD$ : for deterministic batching, $\epsP$ : upper bounds when using Poisson subsampling (computed using different accountants), and $\epsS$ : a lower bound when using shuffling. We observe that shuffling does not provide much amplification for small values of $\sigma$, incurring significantly higher privacy cost  compared to Poisson subsampling.}
\label{fig:eps-vs-sigma-1}
\end{figure}

$\DPSGD$ operates by processing the training data in mini-batches, and at each step, performs a first-order gradient update, using a noisy estimate of the average gradient for each mini-batch. In particular, the gradient $g$ for each record is first clipped to have a pre-determined bounded $\ell_2$-norm, by setting $[g]_C := g \cdot \min\{1, C / \|g\|_2\}$, and then adding Gaussian noise of standard deviation $\sigma C$ to all coordinates of the sum of gradients in the mini-batch.\footnote{While some distributed training setup uses the \emph{sum} gradient directly, it is common to rescale the \emph{sum} gradient by the batch size to obtain the \emph{average} gradient before applying the optimization step. Since this scaling factor can be assimilated in the learning rate, we focus on the \emph{sum} gradient for simplicity in this paper. Note that, in case of $\DPSGD$ using Poisson subsampling, the scaling is done by the {\em expected} batch size, and not the realized batch size.} The privacy guaranteed by the mechanism depends on the following: the choice of $\sigma$, the size of dataset, the size of mini-batches, the number of steps of gradient update performed, and finally {\em the process used to generate the batches}.
Almost all deep learning systems  generate fixed-sized batches of data by going over the dataset sequentially. When feasible, a global shuffling of all the examples in the dataset is performed for each training epoch by making a single pass over the dataset. On the other hand, the process analyzed by \citet{abadi16deep} constructs each batch by including each record with a certain probability, chosen i.i.d. However, this leads to variable-sized mini-batches, which is technically challenging to handle in practice. As a result, there is generally a mismatch between the actual training pipeline and the privacy accounting in many applications of $\DPSGD$, with the implicit assumption that this subtle difference is negligible and excusable. However, in this paper, we show that this is not true---in a typical setting, as shown in \Cref{fig:eps-vs-sigma-1}, the privacy loss from the correct accounting is significantly larger than expected.

\begin{algorithm}[t]
\caption{$\ABLQB$: Adaptive Batch Linear Queries}
\label{alg:adaptive-batch-linear-queries}
\begin{algorithmic}
\PARAMETERS Batch sampler $\cB$ using (expected) batch size $b$ and number of batches $T$, noise parameter $\sigma$, and an (adaptive) query method $\cA : (\R^d)^* \times \cX \to \B^d$.
\REQUIRE Dataset $\bx = (x_1, \ldots, x_n)$.
\ENSURE Query estimates $g_1, \ldots, g_T \in \R^d$
\STATE $(S_1, \ldots, S_T) \gets \cB_{b,T}(n)$
\FOR{$t = 1, \ldots, T$}
    \STATE $\psi_t(\cdot) := \cA(g_1, \ldots, g_{t-1}; \cdot)$
    \STATE $g_t \gets \sum_{i \in S_t} \psi_t(x_i) + e_t$ for $e_t \sim \cN(0, \sigma^2 I_d)$
\ENDFOR
\RETURN $(g_1, \ldots, g_T)$
\end{algorithmic}
\end{algorithm}

\begin{algorithm}[t]
\caption{$\cD_{b, T}$: Deterministic Batch Sampler}
\label{alg:deterministic-batch}
\begin{algorithmic}
\PARAMETERS Batch size $b$, number of batches $T$.
\REQUIRE Number of datapoints $n = b \cdot T$.
\ENSURE Seq. of disjoint batches $S_1, \ldots, S_T \subseteq [n]$.
\FOR{$t = 0, \ldots, T-1$}
    \STATE $S_{t+1} \gets \{tb + 1, \ldots, tb + b\}$
\ENDFOR
\RETURN $S_1, \ldots, S_T$
\end{algorithmic}
\end{algorithm}

\paragraph{Adaptive Batch Linear Queries and Batch Samplers.}

Formally, the privacy analysis of $\DPSGD$, especially in the case of non-convex differentiable models, is performed by viewing it as a post-processing of a mechanism performing {\em adaptive batch linear queries} ($\ABLQB$) as defined in \Cref{alg:adaptive-batch-linear-queries} (we use subscript $\cB$ to emphasize the role of the batch sampler), where the linear query $\psi_t(x_i)$ corresponds to the clipped gradient corresponding to record $x_i$. For any $x$, we require that $\psi_t(x) \in \B^d := \{w \in \R^d : \|w\|_2 \le 1\}$.
Note that without loss of generality, we can treat the norm bound $C$ to be $1$ by defining $\psi_t(x) := [g]_C / C$ for the corresponding gradient $g$, and rescaling $g_t$ by $C$ to get back the noisy gradient for $\DPSGD$. 

As mentioned above, a canonical way to generate the mini-batches is to go through the dataset in a fixed deterministic order, and divide the data into mini-batches of a fixed size~(\Cref{alg:deterministic-batch}, denoted as $\cD$). Another commonly used option is to first randomly permute the entire dataset before dividing it into mini-batches of a fixed size~(\Cref{alg:shuffle-batch}, denoted as $\cS$); this option provides {\em amplification by shuffling}, namely, that the privacy guarantees are better compared to fixed deterministic ordering. However, obtaining such amplification bounds is non-trivial, and while some bounds have been recently established for such privacy amplification~\citep{erlingsson19amplification,feldman21hiding,feldman23stronger}, they tend to be loose in our setting and only kick in when the basic mechanism is already sufficiently private.%

\begin{algorithm}[t]
\caption{$\cS_{b,T}$: Shuffle Batch Sampler}
\label{alg:shuffle-batch}
\begin{algorithmic}
\PARAMETERS Batch size $b$, number of batches $T$.
\REQUIRE Number of datapoints $n = b \cdot T$.
\ENSURE Seq. of disjoint batches $S_1, \ldots, S_T \subseteq [n]$.
\STATE Sample a random permutation $\pi$ over $[n]$.
\FOR{$t = 0, \ldots, T-1$}
    \STATE $S_{t+1} \gets \{\pi(tb + 1), \ldots, \pi(tb + b)\}$
\ENDFOR
\RETURN $S_1, \ldots, S_T$
\end{algorithmic}
\end{algorithm}

\begin{algorithm}[t]
\caption{$\cP_{b,T}$: Poisson Batch Sampler}
\label{alg:poisson-batch}
\begin{algorithmic}
\PARAMETERS Expected batch size $b$, number of batches $T$.
\REQUIRE Number of datapoints $n$.
\ENSURE Seq. of batches $S_1, \ldots, S_T \subseteq [n]$.
\FOR{$t = 1, \ldots, T$}
    \STATE $S_{t} \gets \emptyset$
    \FOR{$i = 1, \ldots, n$}
        \STATE $S_{t} \gets \begin{cases}
            S_{t} \cup \{i\} & \text{ with probability } b/n\\
            S_{t} & \text{ with probability } 1 - b/n\\
        \end{cases}$
    \ENDFOR
\ENDFOR
\RETURN $S_1, \ldots, S_T$
\end{algorithmic}
\end{algorithm}

Instead, the approach followed by \citep{abadi16deep}, and henceforth used commonly in reporting privacy parameters for $\DPSGD$, is to assume that each batch is sampled i.i.d. by including each record with a certain probability, referred to as {\em Poisson subsampling}~(\Cref{alg:poisson-batch}, denoted as $\cP$). The advantage of this approach is that the privacy analysis of such sampling is easier to carry out since the $\ABLQP$ mechanism can be viewed as a composition of $T$ independent sub-mechanisms. This enables privacy accounting methods such as R\'enyi DP~\cite{mironov17renyi} as well as numerically tight accounting methods using privacy loss distributions (elaborated on later in \Cref{sec:dominating-pairs}).

This has been the case, even when the algorithm being implemented in fact uses some form of shuffling-based batch sampler. To quote \citet{abadi16deep} (with our emphasis),
\begin{quote}
``In practice, for efficiency, the construction of batches and lots is done by randomly permuting the examples and then \textsl{partitioning} them into groups of the appropriate sizes. \textsl{For ease of analysis}, however, we assume that \textsl{each lot is formed by independently picking each example with probability $q = L/N$}, where $N$ is the size of the input dataset.''
\end{quote}
Most implementations of $\DPSGD$ mentioned earlier also use some form of shuffling, with a rare exception of PyTorch Opacus~\cite{yousefpour21opacus} that has the option of Poisson subsampling to be consistent with the privacy analysis. But this approach does not scale to large datasets as random access for datasets that do not fit in memory is generally inefficient. Moreover, variable batch sizes are inconvenient to handle in deep learning systems\footnote{For example, when the input shape changes, \texttt{jax.jit} will trigger recompilation, and \texttt{tf.function} will retrace the graph. Google TPUs require all operations to have fixed (input and output) shapes. Moreover, in various form of data parallelism, the batch size needs to be divisible by the number of accelerators.}.
\citet{tf_privacy_statement} provides the \texttt{compute\_dp\_sgd\_privacy\_statement} method for computing the privacy parameters and reminds users about the implicit assumption of Poisson subsampling.
\citet{ponomareva23dpfy} note in their survey that ``It is common, though inaccurate, to train without Poisson subsampling, but to report the stronger DP bounds as if amplification was used.''

As $\DPSGD$ is being deployed in more applications with such discrepancy, it has become crucial to understand exactly how the privacy guarantee depends on the precise choice of the batch sampler, especially when one cares about specific $(\eps, \delta)$ privacy parameters (and not asymptotic bounds).
This leads us to our main motivating question:
\begin{quote}
{\sl How do the privacy guarantees of $\ABLQB$ compare when using different batch samplers $\cB$?}
\end{quote}

\subsection{Our Contributions}

We study the privacy guarantees of $\ABLQB$ for different choices of batch samplers $\cB$. While we defer the formal definition of $(\eps, \delta)$-$\DP$ to \Cref{sec:prelims}, let $\deltaB(\eps)$ denote the {\em privacy loss curve} of $\ABLQB$ for any $\cB \in \{\cD, \cP, \cS \}$, for a fixed choice of $\sigma$ and $T$. Namely, for all $\eps > 0$, let $\deltaB(\eps)$ be the smallest $\delta \ge 0$ such that $\ABLQB$ satisfies $(\eps, \delta)$-$\DP$ for any underlying adaptive query method $\cA$. Let $\epsB(\delta)$ be defined analogously.

\begin{description}[leftmargin=1mm,topsep=0pt,itemsep=0pt]
\item[\boldmath $\cD$ vs $\cS$.]
We observe that $\ABLQS$ always satisfies stronger privacy guarantees than $\ABLQD$, i.e., $\deltaS(\eps) \le \deltaD(\eps)$ for all $\eps \ge 0$.

\item[\boldmath $\cD$ vs $\cP$.]
We show that the privacy guarantee of $\ABLQD$ and $\ABLQP$ are incomparable.
Namely, for all values of $T$ (number of steps) and $\sigma$ (noise parameter), it holds (i) for small enough $\eps > 0$ that $\deltaP(\eps) < \deltaD(\eps)$, but perhaps more interestingly, (ii) for sufficiently large $\eps$ it holds that $\deltaP(\eps) \gg \deltaD(\eps)$.
We also demonstrate this separation in a specific numerical setting of parameters.

\item[\boldmath $\cS$ vs $\cP$.] By combining the above it follows that for sufficiently large $\eps$, it holds that $\deltaS(\eps) < \deltaP(\eps)$. If $\deltaS(\eps) < \deltaP(\eps)$ were to hold for all $\eps > 0$, then reporting privacy parameters for $\ABLQP$ would provide correct, even if pessimistic, privacy guarantees for $\ABLQS$.
However, we demonstrate multiple concrete settings of parameters, for which $\deltaP(\eps) \ll \deltaS(\eps)$, or alternately, $\epsP(\delta) \ll \epsS(\delta)$.

For example, in \Cref{fig:eps-vs-sigma-1}, we fix $\delta = 10^{-6}$ and the number of steps $T = 10,000$,\footnote{Recall that since we are analyzing a ``single epoch'', the subsampling probability of Poisson subsampling is $b/n = 1/T$.} and compare the value of $\epsB(\delta)$ for various values of $\sigma$. For $\sigma = 0.5$, we find $\epsP(\delta) < 1.96$ (PLD) and $\epsP(\delta) < 3.43$ (RDP), but $\epsS(\delta) > 10.994$ and $\epsD(\delta) \approx 10.997$. For $\sigma = 1.3$, we find $\epsP(\delta) < 0.031$ (PLD), whereas, $\epsS(\delta) > 0.26$.

This suggests that reporting privacy guarantees using the Poisson batch sampler can significantly underestimate the privacy loss when the implementation in fact uses the shuffle batch sampler.
\end{description}

Our main takeaway is that batch sampling plays a crucial in determining the privacy guarantees of $\ABLQB$, and hence caution must be exercised in reporting privacy parameters for mechanisms such as $\DPSGD$.

\subsection{Technical Overview}\label{subsec:technical-overview}

Our techniques relies on the notion of {\em dominating pairs} as defined by \citet{zhu22optimal} (see \Cref{def:dominating-pair}), which if {\em tightly dominating} captures the privacy loss curve $\deltaB(\eps)$.

\begin{description}[leftmargin=*,topsep=0pt,itemsep=0pt]
\item[{\boldmath $\cD$ vs $\cS$.}]
We observe that applying a mechanism on a random permutation of the input dataset does not degrade its privacy guarantees.
While standard, we include a proof for completeness.%

\item[{\boldmath $\cD$ vs $\cP$.}]
In order to show that $\deltaP(\eps) < \deltaD(\eps)$ for small enough $\eps > 0$, we first show that $\deltaP(0) < \deltaD(0)$ by showing that the total variation distance between the tightly dominating pair for $\ABLQD$ is larger than that in the case of $\ABLQP$. And thus, by the continuity of hockey stick divergence in $\eps$, we obtain the same for small enough $\eps$.

In order to show that $\deltaP(\eps) > \deltaD(\eps)$ for large enough $\eps > 0$, we demonstrate an explicit set $E$ (a halfspace), which realizes a lower bound on the hockey stick divergence between the tightly dominating pair of $\ABLQP$, and show that this decays slower than the hockey stick divergence between the tightly dominating pair for $\ABLQD$.

We demonstrate the separation in a specific numerical setting of parameters using the \texttt{dp\_accounting} library \cite{GoogleDP} to provide lower bounds on the hockey stick divergence between the tightly dominating pair for $\ABLQP$.

\item[{\boldmath $\cS$ vs $\cP$.}]
One challenge in understanding the privacy guarantee of $\ABLQS$ is the lack of a clear dominating pair for the mechanism. Nevertheless, we consider a specific instance of the query method $\cA$, and a specific adjacent pair $\bx \sim \bx'$. The key insight we use in constructing this $\cA$ and $\bx, \bx'$ is that in the case of $\ABLQS$, since the batches are of a fixed size, the responses to queries on batches containing only the non-differing records in $\bx$ and $\bx'$ can {\em leak} information about the location of the differing record in the shuffled order. This limitation is not faced by $\ABLQP$, since each record is independently sampled in each batch.
We show a lower bound on the hockey stick divergence between the output distribution of the mechanism on these adjacent inputs, by constructing an explicit set $E$, thereby obtaining a lower bound on $\deltaS(\eps)$.

In order to show that this can be significantly larger than the privacy guarantee of $\ABLQP$, we again use the \texttt{dp\_accounting} library, this time to provide upper bounds on the hockey stick divergence between the dominating pair of $\ABLQP$.

\end{description}

We provide the IPython notebook\footnote{\scriptsize \url{https://colab.research.google.com/drive/1zI2H8YEXbQyD6gZVVskFwcOiM5YMvqRe?usp=sharing}} that was used for all the numerical demonstrations; the notebook can be executed using a free CPU runtime on \citet{googlecolab}.

\paragraph{Related work.} The phenomenon of {\em non-differing} records leaking information about whether the differing record is in a batch or not can also exist in sampling-based batch samplers, such as, sampling independent batches of fixed-size, as studied in a concurrent work by \citet{lebeda2024avoiding}. However, we focus on the shuffle-based batch sampler, as these are most common in practical implementations.%

\section{Differential Privacy}\label{sec:prelims}

We consider mechanisms that map input datasets to distributions over an output space, namely $\cM : \cX^* \to \Delta_{\cO}$. That is, on input {\em dataset} $\bx = (x_1, \ldots, x_n)$ where each {\em record} $x_i \in \cX$, $\cM(\bx)$ is a probability distribution over the output space $\cO$; we abuse notation to also use $\cM(\bx)$ to denote the underlying random variable.
Two datasets $\bx$ and $\bx'$ are said to be {\em adjacent}, denoted $\bx \sim \bx'$, if, loosely speaking, they ``differ in one record''. This can be formalized in multiple ways, which we elaborate on in \Cref{subsec:adjacency}, but for any notion of adjacency, Differential Privacy (DP) can be defined as follows.
\begin{definition}[DP]\label{def:dp}
For $\eps, \delta \ge 0$, a mechanism $\cM$ satisfies $(\eps, \delta)$-$\DP$ if for all ``adjacent'' datasets $\bx \sim \bx'$, and for any (measurable) event $E$ it holds that
\[
\Pr[\cM(\bx) \in E] ~\le~ e^{\eps} \Pr[\cM(\bx') \in E] + \delta.
\]
\end{definition}
For any mechanism $\cM$, let $\delta_{\cM} : \R_{\ge 0} \to [0, 1]$ be its {\em privacy loss curve}, namely $\delta_\cM(\eps)$ is the smallest $\delta$ for which $\cM$ satisfies $(\eps, \delta)$-$\DP$; $\eps_{\cM} : [0, 1] \to \R_{\ge 0}$ can be defined analogously.

\subsection{Adaptive Batch Linear Queries Mechanism}\label{subsec:ablq}

We primarily study the {\em adaptive batch linear queries} mechanism $\ABLQB$ using a {\em batch sampler} $\cB$ and an {\em adaptive query method $\cA$}, as defined in \Cref{alg:adaptive-batch-linear-queries}.
The batch sampler $\cB$ can be instantiated with any algorithm that produces a (randomized) sequence $S_1, \ldots, S_T \subseteq [n]$ of batches, where $n$ is the number of examples. $\ABLQB$ produces a sequence $(g_1, \ldots, g_T)$ of responses where each $g_i \in \R^d$. The response $g_t$ is produced recursively using the adaptive query method $\cA$ that given $g_1, \ldots, g_{t-1}$, constructs a new query $\psi_t : \cX \to \B^d$ (for $\B^d := \{ v \in \R^d : \|v\|_2 \le 1\}$), and we estimate the sum of $\psi_t(x)$ over the batch $S_t$ with added zero-mean Gaussian noise of scale $\sigma$ to all coordinates. As explained in \Cref{sec:intro}, $\DPSGD$ falls under this abstraction by considering an adaptive query method that is specified by a differentiable loss function $f : \R^d \times \cX \to \R$, and starts with an initial $w_0 \in \R^d$, and defines the query $\cA(g_1, \ldots, g_{t-1}; x)$ as the clipped gradient $[\nabla_w f_{w_{t-1}}(x)]_1$ where $w_t$ is the $t$th model iterate recursively obtained by performing gradient descent, e.g., $w_t \gets w_{t-1} - \eta_t g_t$ (or any other first-order optimization step).

We consider the Deterministic $\cD$ (\Cref{alg:deterministic-batch}), Poisson $\cP$ (\Cref{alg:poisson-batch}) and Shuffle $\cS$ (\Cref{alg:shuffle-batch}) batch samplers. As used in \Cref{sec:intro}, we will continue to use $\deltaB(\eps)$ as a shorthand for denoting the {\em privacy loss curve} of $\ABLQB$ for any $\cB \in \{\cD, \cP, \cS \}$. Namely, for all $\eps > 0$, let $\deltaB(\eps)$ be the smallest $\delta \ge 0$ such that $\ABLQB$ satisfies $(\eps, \delta)$-$\DP$ for {\em all} choices of the underlying adaptive query method $\cA$. And $\epsB(\delta)$ is defined analogously.

\subsection{Adjacency Notions}\label{subsec:adjacency}
As alluded to earlier, the notion of adjacency is crucial to \Cref{def:dp}. Commonly used adjacency notions are:
\begin{description}[leftmargin=0pt,topsep=0pt,itemsep=0pt]
\item [Add-Remove adjacency.] Datasets $\bx, \bx' \in \cX^*$ are said to be {\em add-remove} adjacent if there exists an $i$ such that $\bx' = \bx_{-i}$ or vice-versa (where $\bx_{-i}$ represents the dataset obtained by removing the $i$th record in $\bx$).
\item [Substitution adjacency.] Datasets $\bx, \bx' \in \cX^*$ are said to be {\em substitution} adjacent if there exists an $i$ such that $\bx'_{-i} = \bx_{-i}$ and $x_i' \ne x_i$. 
\end{description}

The privacy analysis of $\DPSGD$ is typically done for the Poisson batch sampler $\cP$ \cite{abadi16deep,mironov17renyi}, with respect to the {\em add-remove} adjacency. However, it is impossible to analyze the privacy of $\ABLQD$ or $\ABLQS$ with respect to the {\em add-remove} adjacency because the batch samplers $\cD$ and $\cS$ require that the number of records $n$ equals $b \cdot T$. On the other hand, using the {\em substitution} adjacency for $\cD$ and $\cS$ leads to an unfair comparison to $\ABLQP$ whose analysis is with respect to the {\em add-remove} adjacency. Thus, to make a fair comparison, we consider the following adjacency (proposed by \citet{kairouz21practical}).

\begin{description}[leftmargin=0pt,topsep=0pt,itemsep=0pt]
\item [Zero-out adjacency.] We augment the input space to be $\cX_{\bot} := \cX \cup \{\bot\}$ and extend any adaptive query method $\cA$ as $\cA(g_1, \ldots, g_t; \bot) = \mathbf{0}$ for all $g_1, \ldots, g_t \in \R^d$. Datasets $\bx, \bx' \in \cX_{\bot}^n$ are said to be {\em zero-out} adjacent if there exists $i$ such that $\bx_{-i} = \bx'_{-i}$, and exactly one of $\{x_i, x_i'\}$ is in $\cX$ and the other is $\bot$. Whenever we need to specifically emphasize that $x_i \in \cX$ and $x'_i = \bot$, we will denote it as $\bx \to_z \bx'$. In this notation, $\bx \sim \bx'$ if either $\bx \to_z \bx'$ or $\bx' \to_z \bx$.
\end{description}

The privacy analysis of $\ABLQP$ with respect to zero-out adjacency is the same as that with respect to the add-remove adjacency; it is essentially replacing a record by a ``ghost'' record that makes the query method always return $\mathbf{0}$. In the rest of this paper, we only consider this zero-out adjacency.

We note that our separations where $\epsP(\delta) \ll \epsS(\eps)$, such as in \Cref{fig:eps-vs-sigma-1}, also hold under the substitution adjacency, by using ``group privacy'', namely if a mechanism satisfies $(\eps, \delta)$-$\DP$ with respect to zero-out adjacency, then it satisfies $(2\eps, \delta \cdot (1+e^{\eps}))$-$\DP$ (see, e.g., \citep{vadhan17complexity}).%

\subsection{Hockey Stick Divergence}\label{subsec:pld}
We interchangeably use the same notation (e.g., letters such as $P$) to denote both a probability distribution and its corresponding density function. For $\mu \in \R^D$ and positive semi-definite $\Sigma \in \R^{D \times D}$, we use $\cN(\mu, \Sigma)$ to denote the Gaussian distribution with mean $\mu$ and covariance $\Sigma$.
For probability densities $P$ and $Q$, we use $\alpha P + \beta Q$ to denote the weighted sum of the corresponding densities.
$P \otimes Q$ denotes the product distribution sampled as $(u, v)$ for $u \sim P$, $v \sim Q$, and, $P^{\otimes T}$ denotes the $T$-fold product distribution $P \otimes \cdots \otimes P$.
\begin{definition}\label{def:hockey-stick-divergence}
For all $\eps \in \R$, the {\em $e^\eps$-hockey stick divergence} between $P$ and $Q$ is
$D_{e^{\eps}}(P \| Q) := \sup_E \{ P(E) - e^{\eps} Q(E) \}$.
\end{definition}

It is immediate to see that $\cM$ satisfies $(\eps, \delta)$-$\DP$ iff for all adjacent $\bx \sim \bx'$, it holds that $D_{e^{\eps}}(\cM(\bx) \| \cM(\bx')) \le \delta$.

\begin{definition}[Dominating Pair~\cite{zhu22optimal}]\label{def:dominating-pair}
The pair $(P, Q)$ {\em dominates} the pair $(A, B)$ if $D_{e^{\eps}}(P \| Q) ~\ge~ D_{e^{\eps}}(A \| B)$ holds for all $\eps \in \R$.
We say that $(P, Q)$ {\em dominates} a mechanism $\cM$ if $(P, Q)$ dominates $(\cM(\bx), \cM(\bx'))$ for all adjacent $\bx \to_z \bx'$.
\end{definition}
If $(P, Q)$ dominates $\cM$, then for all $\eps \ge 0$, $\delta_{\cM}(\eps) \le \max \{ D_{e^\eps}(P\|Q), D_{e^{\eps}}(Q \| P) \}$.
We say that $(P, Q)$ {\em tightly dominates} a mechanism $\cM$ if $(P, Q)$ dominates $\cM$ and there exist adjacent datasets $\bx \to_z \bx'$ such that $D_{e^{\eps}}(P \| Q) = D_{e^{\eps}}(\cM(\bx) \| \cM(\bx'))$ holds for all $\eps \in R$ (note that this includes $\eps < 0$); in this case, $\delta_{\cM}(\eps) = \max\{ D_{e^\eps}(P\|Q), D_{e^{\eps}}(Q \| P) \}$.
Thus, tightly dominating pairs completely characterize the privacy loss of a mechanism (although they are not guaranteed to exist for all mechanisms).\footnote{\citet{zhu22optimal} define ``tightly dominating pair'' differently, in a manner that is guaranteed to exist. They additional define the notion of a ``worst-case pair'', which is a pair of adjacent datasets $\bx \sim \bx'$ such that $(\cM(\bx), \cM(\bx'))$ is a tightly dominating pair. Thus, our notion of ``tightly dominating pair'' refers precisely to the pair $(\cM(\bx), \cM(\bx'))$ for a worst-case adjacent pair $\bx$, $\bx'$. It is also worth noting that our notation for ``tightly dominating pairs'' is asymmetric as it only considers pairs $\bx \to_{z} \bx'$; the reverse setting is handled implicitly because $(P, Q)$ dominates $(\cM(\bx), \cM(\bx'))$ if and only if $(Q, P)$ dominates $(\cM(\bx'), \cM(\bx))$.}
Dominating pairs behave nicely under mechanism compositions. Namely, if $(P_1, Q_1)$ dominates $\cM_1$ and $(P_2, Q_2)$ dominates $\cM_2$, then $(P_1 \otimes P_2, Q_1 \otimes Q_2)$ dominates the (adaptively) composed mechanism $\cM_1 \circ \cM_2$.

\section{\boldmath Dominating Pairs for \texorpdfstring{$\ABLQB$}{ABLQ\_B}}\label{sec:dominating-pairs}

We discuss the dominating pairs for $\ABLQB$ for $\cB \in \{\cD, \cP, \cS\}$ that will be crucial for establishing our results.

\paragraph{\boldmath Tightly dominating pair for $\ABLQD$.} It follows from the standard analysis of the Gaussian mechanism and parallel composition that a tightly dominating pair for $\ABLQD$ is the pair $(\PD := \cN(1, \sigma^2), \QD := \cN(0, \sigma^2))$, leading to a closed-form expression for $\deltaD(\eps)$.

\begin{proposition}[Theorem 8 in \citet{balle18improving}]\label{prop:D-hockey}
For all $\eps \ge 0$, it holds that
\[
\textstyle\deltaD(\eps) = \Ncdf\prn{- \sigma \eps + \frac{1}{2\sigma}} - e^{\eps} \Ncdf\prn{- \sigma \eps - \frac{1}{2\sigma}},
\]
where $\Ncdf(\cdot)$ is the cumulative density function (CDF) of the standard normal random variable $\cN(0, 1)$.
\end{proposition}

\paragraph{\boldmath Tightly dominating pair of $\ABLQP$.} \citet{zhu22optimal} showed\footnote{This was implicit in prior work, e.g., \cite{koskela2020computing}.} that the tightly dominating pair for a single step of $\ABLQP$, a Poisson sub-sampled Gaussian mechanism, is given by the pair $(A = (1 - q) \cN(0, \sigma^2) + q \cN(1, \sigma^2), B = \cN(0, \sigma^2))$, where $q$ is the sub-sampling probability of each record, namely $q = b/n$, and in the case where $n = b \cdot T$, we have $q = 1/T$. Since $\ABLQP$ is a $T$-fold composition of this Poisson subsampled Gaussian mechanism, it follows that the tightly dominating pair for $\ABLQP$ is $(\PP := A^{\otimes T}, \QP := B^{\otimes T})$.

The hockey stick divergence $D_{e^{\eps}}(\PP \| \QP)$ does not have a closed-form expression, but there are privacy accountants based on the methods of R\'enyi DP (RDP) ~\cite{mironov17renyi} as well as {\em privacy loss distributions (PLD)} ~\cite{meiser2018tight, sommer2019privacy}, the latter providing numerically accurate algorithms~\cite{koskela2020computing,gopi21numerical,ghazi22faster,doroshenko22connect},
and have been the basis of multiple open-source implementations from both industry and academia including \citep{DPBayes,GoogleDP,MicrosoftDP}.
While R\'enyi-DP-based accounting provides an upper bound on $\max\{D_{e^{\eps}}(\PP \| \QP), D_{e^{\eps}}(\QP \| \PP)\}$, the PLD-based accounting implementations can provide upper and lower bounds on $\max\{D_{e^{\eps}}(\PP \| \QP), D_{e^{\eps}}(\QP \| \PP)\}$ to high accuracy, as controlled by a certain discretization parameter.

\paragraph{\boldmath Tightly dominating pair for $\ABLQS$?}
It is not clear which adjacent pair would correspond to a tightly dominating pair for $\ABLQS$, and moreover, it is even a priori unclear if one even exists. However, in order to prove {\em lower bounds} on the privacy parameters, it suffices to consider a specific instantiation of the adaptive query method $\cA$, and a particular adjacent pair $\bx \sim \bx'$. In particular, we instantiate the query method $\cA$ as follows.

Consider the input space $\cX = [-1, 1]$, and assume that the query method $\cA$ is non-adaptive, and always produces the query $\psi_t(x) = x$. We consider the adjacent datasets:
\begin{itemize}[leftmargin=*,topsep=0pt,itemsep=0pt]
\item $\bx = (x_1=-1, \ldots, x_{n-1}=-1, x_n=1)$, and
\item $\bx' = (x_1=-1, \ldots, x_{n-1}=-1, x_n=\bot)$.
\end{itemize}

In this case, when the differing record falls in batch $t$, then output of the mechanism on all batches $t' \ne t$ is centered at $-b$, and is centered at $-b+2$ (on input $\bx$) or at $-b+1$ (on input $\bx'$) on batch $t$. Thus it follows that the distributions $A = \ABLQS(\bx)$ and $B = \ABLQS(\bx')$ are given as:
\begin{align*}
A &~=~\textstyle \sum_{t=1}^T \frac{1}{T} \cdot \cN(-b \cdot \mathbf{1} + 2 e_t, \sigma^2 I), \\
B &~=~\textstyle \sum_{t=1}^T \frac{1}{T} \cdot \cN(-b \cdot \mathbf{1} + e_t, \sigma^2 I),
\end{align*}
where $\mathbf{1}$ denotes the all-$1$'s vector in $\R^T$ and $e_t$ denotes the $t$th standard basis vector in $\R^T$.
Shifting the distributions by $b \cdot \mathbf{1}$ does not change the hockey stick divergence $D_{e^{\eps}}(A \| B)$, hence we might as well consider the pair
\begin{align}
\PS &:=~\textstyle \sum_{t=1}^T \frac{1}{T} \cdot \cN(2 e_t, \sigma^2 I)\label{eq:shuffle-upper},\\
\QS &:=~\textstyle \sum_{t=1}^T \frac{1}{T} \cdot \cN(e_t, \sigma^2 I).\label{eq:shuffle-lower}
\end{align}
Thus, $\deltaS(\eps) \ge \max\{D_{e^{\eps}}(\PS \| \QS), D_{e^{\eps}}(\QS \| \PS)\}$.
We conjecture that this pair is in fact tightly dominating for $\ABLQS$ for all instantiations of query methods $\cA$ (including adaptive ones). We elaborate more in \Cref{subsec:intuition}.%
\begin{conjecture}\label{conj:shuffle-dominating}
The pair $(\PS, \QS)$ tightly dominates $\ABLQS$ for all adaptive query methods $\cA$.
\end{conjecture}

The results in this paper do \emph{not} rely on this conjecture being true, as we only use the dominating pair $(\PS, \QS)$ to establish {\em lower bounds} on $\deltaS(\cdot)$.

\section{Privacy Loss Comparisons}\label{sec:separations}

\subsection{\boldmath \texorpdfstring{$\ABLQD$}{ABLQ\_D} vs \texorpdfstring{$\ABLQS$}{ABLQ\_S}}\label{sec:D-vs-S}

We first note that $\ABLQS$ enjoys stronger privacy guarantees than $\ABLQD$.
\begin{theorem}\label{thm:D-vs-S}
For all $\sigma, \eps \ge 0$ and $T \ge 1$: $\deltaS(\eps) \le \deltaD(\eps)$.
\end{theorem}
This follows from a standard technique that shuffling cannot degrade the privacy guarantee satisfied by a mechanism. For completeness, we provide a proof in \Cref{apx:D-vs-S}.%

\subsection{\boldmath \texorpdfstring{$\ABLQD$}{ABLQ\_D} vs \texorpdfstring{$\ABLQP$}{ABLQ\_P}}\label{sec:D-vs-P}

We show that $\ABLQD$ and $\ABLQP$ have incomparable privacy loss. In particular, we show the following.
\begin{theorem}\label{thm:D-vs-P-separation}
For all $\sigma > 0$ and $T > 1$, there exist $\eps_0, \eps_1 \ge 0$ such that,
\begin{enumerate}[label=(\alph*),topsep=0pt,itemsep=0pt]
\item $\forall \eps \in [0, \eps_0)$, it holds that $\deltaD(\eps) > \deltaP(\eps)$, and
\item $\forall \eps > \eps_1$, it holds that $\deltaD(\eps) < \deltaP(\eps)$.
\end{enumerate}
\end{theorem}

We defer the detailed proof to \Cref{apx:D-vs-P}, and provide a proof sketch here. Part (a) is shown by first establishing that the total variation distance (corresponds to $D_1(\cdot \| \cdot)$) between $\PD$ and $\QD$ is strictly larger than the total variation distance between $\PP$ and $\QP$ when $T > 1$ and $\sigma > 0$. This implies that, $\deltaD(0) > \deltaP(0)$. By using the continuity of $D_{e^\eps}(\cdot \| \cdot)$ in $\eps$, we conclude the same for all $\eps < \eps_0$.

For part (b), we construct an explicit set $E$ such that $\PP(E) - e^{\eps} \QP(E) > \deltaD(\eps)$. In particular, we choose a halfspace $E := \left\{ w \in \R^T \big | \sum_i w_i > (\eps + \log 2 + T \log T) \sigma^2 + {\frac{T}{2}}\right\}$ and show that $\PP(E) - e^{\eps} \QP(E)$ is at least $\frac12 \Ncdf\prn{- \frac{\eps\sigma}{\sqrt{T}} - \frac{(T\log T + \log 2)\sigma}{\sqrt{T}} - \frac{\sqrt{T}}{2\sigma}}$.
For large $\eps$, the dominant term is $-\eps \sigma / \sqrt{T}$.
On other hand, $\deltaD(\eps)$ is at most $\Ncdf(-\eps\sigma + \frac{1}{2\sigma})$ (from \Cref{prop:D-hockey}), which has the dominant term $-\eps \sigma$.
Since $-\eps\sigma/\sqrt{T}$ decays slower than $-\eps\sigma$, we get that for sufficiently large $\eps_1$, it holds that $\deltaD(\eps) < \deltaP(\eps)$ for all $\eps > \eps_1$.

Even though \Cref{thm:D-vs-P-separation} was proved for some values of $\eps_0$ and $\eps_1$, we conjecture that it holds for $\eps_0 = \eps_1$.
\begin{conjecture}\label{conj:D-vs-P-single-threshold}
\Cref{thm:D-vs-P-separation} holds for $\eps_0 = \eps_1$.
\end{conjecture}
We do \emph{not} rely on this conjecture being true in the rest of this paper. We provide a numerical example that validates \Cref{thm:D-vs-P-separation} and provides evidence for \Cref{conj:D-vs-P-single-threshold}. In \Cref{fig:D-vs-P-separation}, for $\sigma = 0.3$ and $T = 10$, we plot the numerically computed $\deltaD(\eps)$ (using \Cref{prop:D-hockey}), as well as lower and upper bounds on $\deltaP(\eps)$, computed using the open source \texttt{dp\_accounting} library~\cite{GoogleDP}.

\begin{figure}
\centering
\includegraphics[width=8cm]{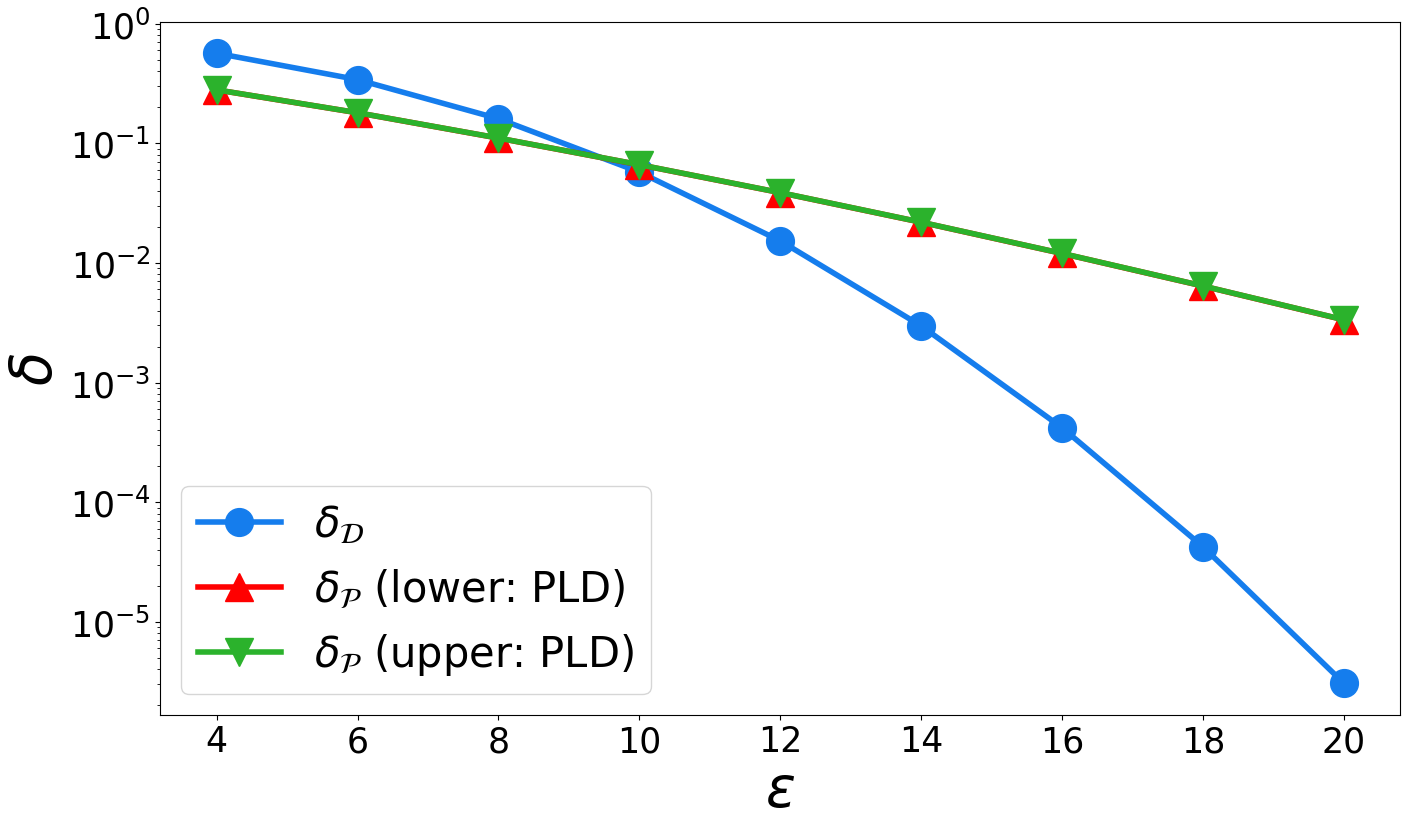}
\caption{$\deltaD(\eps)$ and $\deltaP(\eps)$ for $\sigma=0.3$ and $T = 10$.}
\label{fig:D-vs-P-separation}
\end{figure}

\subsection{\boldmath \texorpdfstring{$\ABLQP$}{ABLQ\_P} vs \texorpdfstring{$\ABLQS$}{ABLQ\_S}}\label{sec:P-vs-S}

From \Cref{thm:D-vs-S,thm:D-vs-P-separation}, it follows that there exists an $\eps_1$ such that for all $\eps > \eps_1$, it holds that $\deltaS(\eps) \le \deltaD(\eps) < \deltaP(\eps)$. On the other hand, while we know that $\deltaD(\eps) > \deltaP(\eps)$ for sufficiently small $\eps$, this does not imply anything about the relationship between $\deltaS(\eps)$ and $\deltaP(\eps)$ for small $\eps$.

We demonstrate simple numerical settings where $\deltaS(\eps)$ is significantly larger than $\deltaP(\eps)$. We prove lower bounds on $\deltaS(\eps)$ by constructing specific sets $E$ and using the fact that $\deltaS(\eps) \ge \PS(E) - e^{\eps} \QS(E)$.

In particular, we consider sets $E_C$ parameterized by $C$ of the form $\{ w \in \R^T : \max_{t} w_t \ge C \}$; note that $E_C$ is the complement of $T$-fold Cartesian product of the set $(-\infty, C)$.
For a single Gaussian distribution $D = \cN(\mu, \sigma^2 I)$, we can compute the probability mass of $E_C$ under measure $D$ as:
\begin{align*}
D(E_C)
&\textstyle~=~ 1 - D(\R^T \smallsetminus E_C) \\
&\textstyle~=~ 1 - \prod_{t=1}^T \Pr_{x \sim \cN(\mu_t, \sigma^2)}[x < C]\\
&\textstyle~=~ 1 - \prod_{t=1}^T \Ncdf\prn{\frac{C - \mu_t}\sigma}.
\end{align*}
In particular, when $\mu$ is $\alpha \cdot e_t$ for any standard basis vector $e_t$, we have
$D(E_C) = 1 - \Ncdf\prn{\frac{C - \alpha}{\sigma}} \cdot \Ncdf\prn{\frac C \sigma}^{T-1}$. Thus, we have that $\PS(E_C)$ is
\begin{align*}
\PS(E_C) &~=~\textstyle \sum_{t=1}^T \frac1T D_t(E_C) \quad \text{for } D_t = N(2e_t, \sigma^2 I)\\
&\textstyle~=~ 1 - \Ncdf\prn{\frac{C - 2}{\sigma}} \cdot \Ncdf\prn{\frac C \sigma}^{T-1}.
\end{align*}
Similarly, we have $\QS(E_C) = 1 - \Ncdf\prn{\frac{C - 1}{\sigma}} \cdot \Ncdf\prn{\frac C \sigma}^{T-1}$.

Thus, we use the following lower bound:
\begin{align}
    \deltaS(\eps) ~\ge~ \max_{C \in \cC} \PS(E_C) - e^{\eps} \QS(E_C)\label{eq:deltaS-lower-bound}
\end{align}
for any suitable set $\cC$ that can be enumerated over. In our experiments described below, we set $\cC$ to be the set of all values of $C$ ranging from $0$ to $100$ in increments of $0.01$.

\begin{figure}
\centering
\includegraphics[width=8cm]{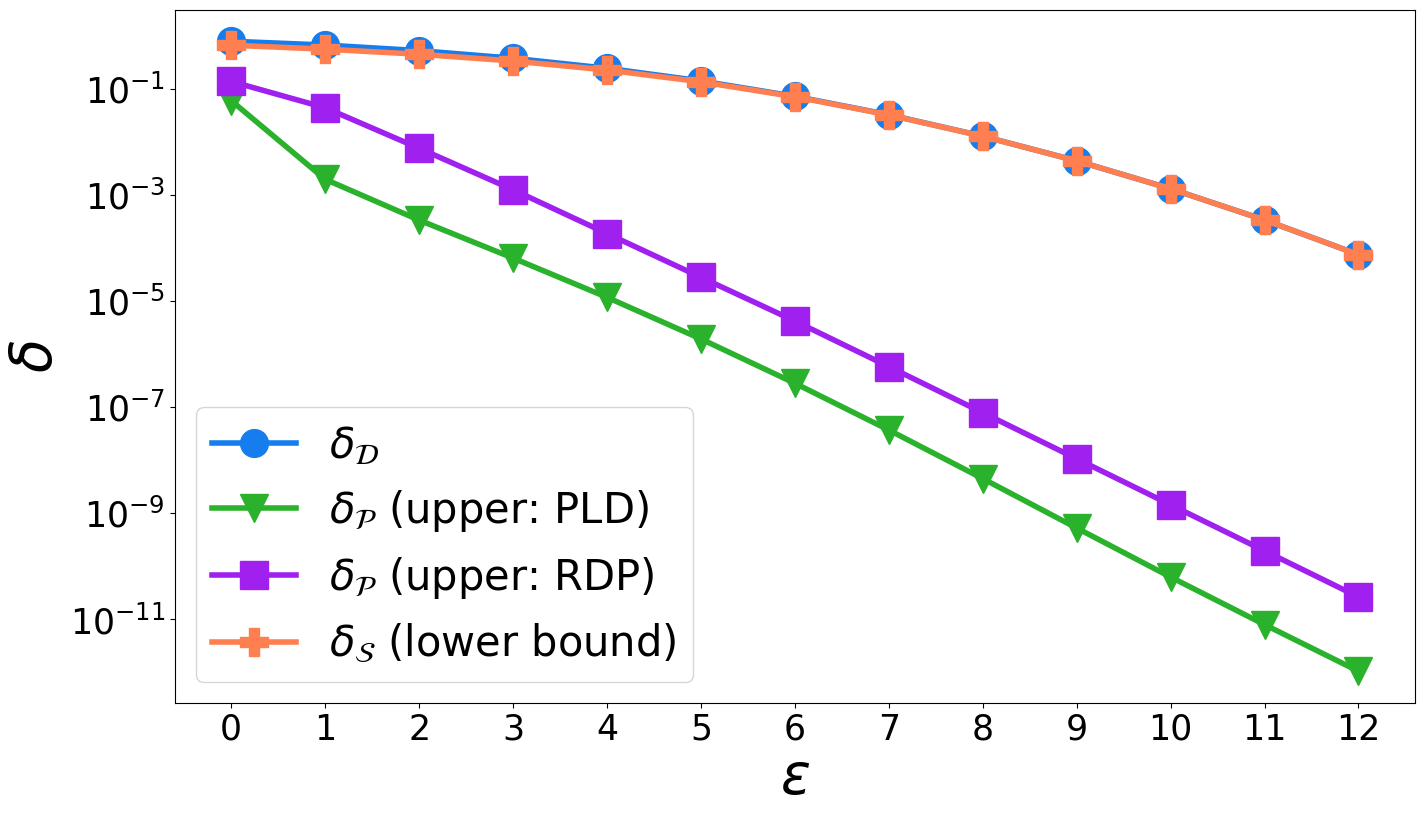}
\caption{$\deltaD(\eps)$, upper bounds on $\deltaP(\eps)$ and a lower bound on $\deltaS(\eps)$ for varying $\eps$ and fixed $\sigma = 0.4$ and $T = 10,000$.}
\label{fig:delta-vs-eps-1}
\end{figure}

In \Cref{fig:delta-vs-eps-1}, we set $\sigma=0.4$ and number of steps $T = 10,000$ and plot $\deltaD(\eps)$, an upper bound on $\deltaP(\eps)$ (obtained using \texttt{dp\_accounting}) and a lower bound on $\deltaS(\eps)$ as obtained via \eqref{eq:deltaS-lower-bound}. We find that while $\deltaP(4) \le 1.18 \cdot 10^{-5}$, $\deltaS(4) \ge 0.226$, that is close to $\deltaD(4) \approx 0.244$. Even $\deltaS(12) \ge 7.5 \cdot 10^{-5}$ is larger than $\deltaP(4)$.
While the $(4, 1.2 \cdot 10^{-5})$-$\DP$ guarantee of $\ABLQP$ could have been considered as sufficiently private, $\ABLQS$ only satisfies much worse privacy guarantees. We provide additional examples in \Cref{apx:experiments}.

\subsubsection{Intuition for \Cref{conj:shuffle-dominating}}\label{subsec:intuition}

We attempt to shed some intuition for why $\ABLQS$ does not provide as much amplification over $\ABLQD$, compared to $\ABLQP$, and why we suggest \Cref{conj:shuffle-dominating}.

For sake of intuition, let's consider the setting where the query method $\cA$ always generates the query $\psi_t(x) = x$, and we have two adjacent datasets:
\begin{itemize}[leftmargin=*,nosep]
\item $\bx = (x_1 = -L, \ldots, x_{n-1}=-L, x_n = 1)$, and
\item $\bx' = (x_1 = -L, \ldots, x_{n-1}=-L, x_n = \bot)$.
\end{itemize}
The case of $L > 1$ is not valid, since in this case $|\psi_t(x)| = L > 1$. However, we can still ask how well the privacy of the $n$th example is preserved by $\ABLQB$, by considering the hockey stick divergence between $\ABLQB(\bx)$ and $\ABLQB(\bx')$.

The crucial difference between $\ABLQP$ and $\ABLQS$ is that the privacy analysis of $\ABLQP$ does not depend at all on the non-differing records in the two datasets.%
In the case of $\ABLQS$, we observe that for any fixed $\sigma$ and $T$, the hockey stick divergence $D_{e^{\eps}}(\ABLQS(\bx) \| \ABLQS(\bx'))$ approaches $\deltaD(\eps)$ as $L \to \infty$. We sketch this argument intuitively: For any batch $S_t$ that does not contain $n$, the corresponding $g_t = -bL + e_t$ for $e_t \sim \cN(0, \sigma^2)$. Whereas for the batch $S_t$ that contains $n$, the corresponding $g_t = -(b-1)L + 1 + e_t$ in case of input $\bx$ or $g_t \sim -(b-1)L + e_t$ in case of input $\bx'$. As $L \to \infty$, we can identify the batch $S_t$ that contains $n$ with probability approaching $1$, thereby not providing any amplification.

In summary, the main differing aspect about $\ABLQS$ and $\ABLQP$ is that in the former, the non-differing examples can leak information about the location of the differing example in the shuffled order, but it is not the case in the latter. While we sketched an argument that works asymptotically as $L \to \infty$, we see glimpses of it already at $L = 1$.

Our intuitive reasoning behind \Cref{conj:shuffle-dominating} is that even in the case of (vector-valued) adaptive query methods, in order to ``leak the most information'' about the differing record $x_i$ between $\bx$ and $\bx'$, it seems natural that the query $\psi_t(x_j)$ should evaluate to $-\psi_t(x_i)$ for all $j \ne i$. If the query method satisfies this condition for all $t$, then it is easy to show that $(\PS, \QS)$ tightly dominates $(\ABLQS(\bx), \ABLQS(\bx'))$. \Cref{conj:shuffle-dominating} then asserts that this is indeed the worst case.

\section{Conclusion \& Future Directions}\label{sec:conclusion}

We identified significant gaps between the privacy analysis of adaptive batch linear query mechanisms, under the deterministic, Poisson, and shuffle batch samplers. We find that while shuffling always provides better privacy guarantees over deterministic batching, Poisson batch sampling can provide a worse privacy guarantee than even deterministic sampling at large $\eps$. But perhaps most surprisingly, we demonstrate that the amplification guarantees due to shuffle batch sampling can be severely limited compared to the amplification due to Poisson subsampling, in various regimes that could be considered of practical interest.

Several interesting directions remain to be investigated. In our work, we provide a technique to only provide a {\em lower bound} on the privacy guarantee when using a shuffle batch sampler. It would be interesting to have a tight accounting method for $\ABLQS$. A first step towards this could be to establish \Cref{conj:shuffle-dominating}, which if true, might make numerical accountants for computing the hockey stick divergence possible. While this involves computing a high-dimensional integral, it might be possible to approximate using importance sampling; e.g., such approaches have been attempted for $\ABLQP$~\cite{wang23randomized}. Also, our approach for analyzing the privacy with shuffle batch sampler is limited to a ``single epoch'' mechanism, whereas in practice, it is common to use $\DPSGD$ with multiple epochs. Extending our approach to multiple epochs will be interesting.%

However, it remains to be seen how the utility of $\DPSGD$ would be affected when we use the correct privacy analysis for $\ABLQS$ instead of the analysis for $\ABLQP$, which has been used extensively so far and treated as a good ``approximation''. Alternative approaches such as $\mathsf{DP}\text{-}\mathsf{FTRL}$~\cite{kairouz21practical,mcmahan22dpmf} that do not rely on amplification might turn out to be better if we instead use the correct analysis for $\ABLQS$, in the regimes where such methods are currently reported to under perform compared to $\DPSGD$.

An important point to note is that the model of shuffle batch sampler we consider here is a simple one. There are various types of data loaders used in practice, which are not necessarily captured by our model. For example \texttt{tf.data.Dataset.shuffle} takes in a parameter of buffer size $b$. It returns a random record among the first $b$ records, and immediately replaces it with the next record ($(b+1)$st in this case), and continues repeating this process. This leads to an asymmetric form of shuffling, when the dataset size exceeds the size of the buffer. Such batch samplers might thus require more careful privacy analysis.

The notion of $\DP$ aims to guarantee privacy even in the worst case. For example in the context of $\DPSGD$, it aims to protect privacy of a single record even when the model trainer and all other records participating in the training are colluding to leak information about this one record. And moreover, releasing the final model is assumed to leak as much information as releasing all intermediate iterates.
Such strong adversarial setting might make obtaining good utility to be difficult.
Alternative techniques for privacy amplification such as amplification by iteration~\cite{feldman18iteration,altschuler22privacy} or through convergence of Langevin dynamics~\cite{chourasia21langevin} have been studied, where only the last iterate of the model is released. However, these analyses rely on additional assumptions such as convexity and smoothness of the loss functions. Investigating whether it is possible to relax these assumptions to make amplification by iteration applicable to non-convex models, even if under some assumptions that are applicable to the ones used in practice, is an interesting future direction.

\section*{Acknowledgments}
We thank the anonymous ICML reviewers for their thoughtful comments and suggestions which have improved the presentation of this paper. We also thank Christian Janos Lebeda, Matthew Regehr, Gautam Kamath, and Thomas Steinke for correspondence about their concurrent work~\cite{lebeda2024avoiding}.

\section*{Impact Statement}

This paper presents work that points out that the particular type of batch sampling used can play a significant role in the privacy analysis of DP-SGD type of algorithms. The impact we see coming from our work is to make practitioners of DP more aware of these gaps. There are many potential societal consequences of DP itself, none which we feel must be specifically highlighted here.

\bibliographystyle{icml2024}
\bibliography{main.bbl}

\newpage
\appendix
\onecolumn
\section{Additional Evaluations}\label{apx:experiments}

\subsection{\boldmath $\eps$ vs. $\sigma$ for fixed $\delta$ and $T$}\label{subsec:eps_vs_sigma}
We first plot $\eps$ against $\sigma$ for fixed $\delta$ and $T$. We compute upper bounds on $\epsP(\delta)$ using R\'enyi DP (RDP) as well as using privacy loss distributions (PLD). These accounting methods are provided in the open source Google \texttt{dp\_accounting} library~\cite{GoogleDP}.

In particular we consistently find that for small values of $\sigma$, $\ABLQS$ provides almost no improvement over $\ABLQD$, and has $\epsS$ that is significantly larger than $\epsP$.
\begin{itemize}[topsep=2mm,leftmargin=4mm,nosep]
\item In \Cref{fig:eps-vs-sigma-2}, we set $\delta=10^{-6}$ and number of steps $T = 100,000$.
In particular, for $\sigma = 0.4$, we find that $\epsP(\delta) \le 3$ (as per PLD accounting) and $\epsP(\delta) \le 4.71$ (as per RDP accounting), whereas on the other hand, $\epsS(\delta) \ge 14.45$, which is very close to $\epsD(\delta)$. For $\sigma = 1.3$, we find that $\epsP(\delta) < 0.01$ (as per PLD accounting), whereas, $\epsS(\delta) > 0.029$.

\begin{figure}[h]
\centering
\includegraphics[width=10cm]{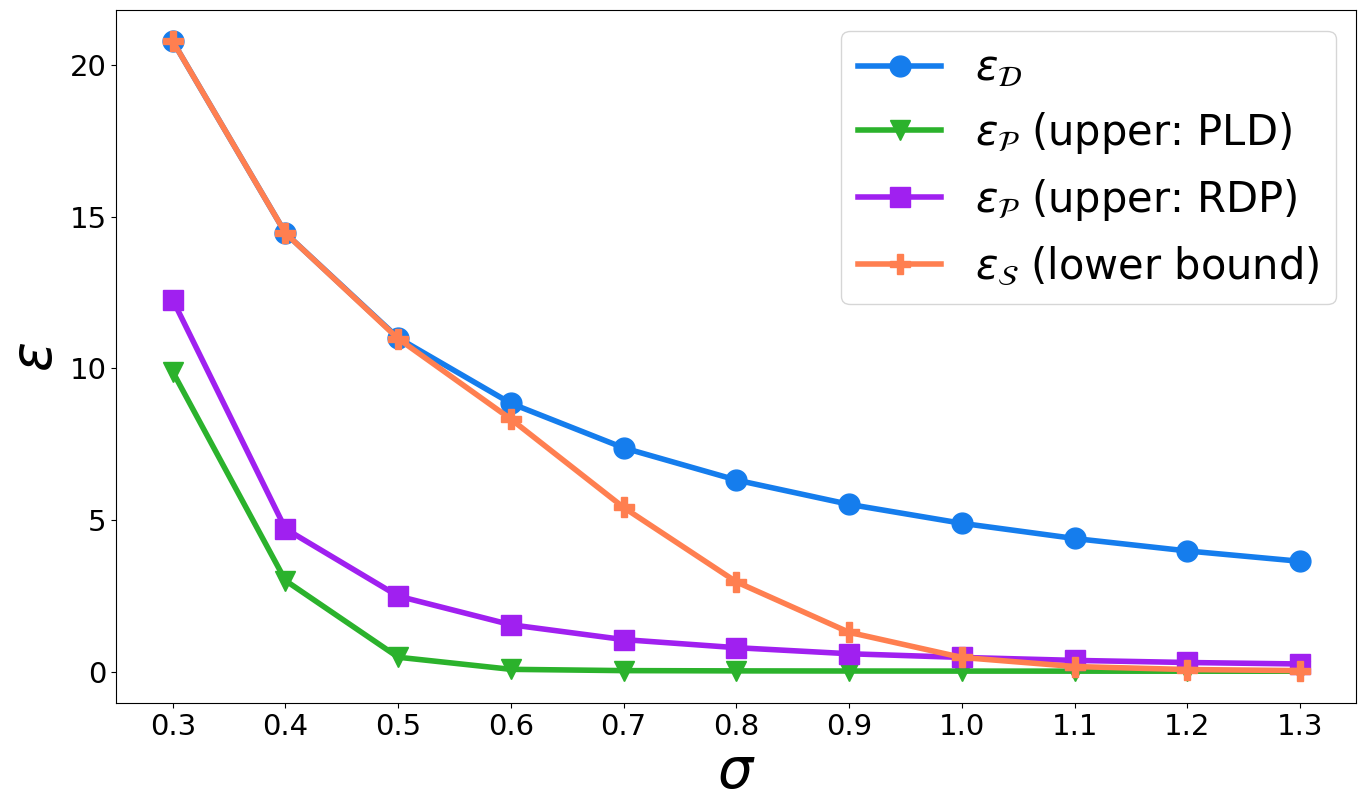}
\caption{$\epsD(\delta)$, upper bounds on $\epsP(\delta)$ and a lower bound on $\epsS(\delta)$ for varying $\sigma$ and fixed $\delta = 10^{-6}$ and $T = 10,000$.}
\label{fig:eps-vs-sigma-2}
\end{figure}

\item In \Cref{fig:eps-vs-sigma-3}, we set $\delta=10^{-5}$ and number of steps $T = 1000$. In particular, for $\sigma = 0.7$, $\epsP(\delta) \le 0.61$ (as per PLD accounting) and $\epsP(\delta) \le 1.64$ (as per RDP accounting), whereas on the other hand, $\epsS(\delta) \ge 6.528$ and $\epsD(\delta) \approx 6.652$. For $\sigma = 1.3$, we find that $\epsP(\delta) < 0.092$ (as per PLD accounting), whereas, $\epsS(\delta) > 0.83$.
\end{itemize}

\begin{figure}[h]
\centering
\includegraphics[width=10cm]{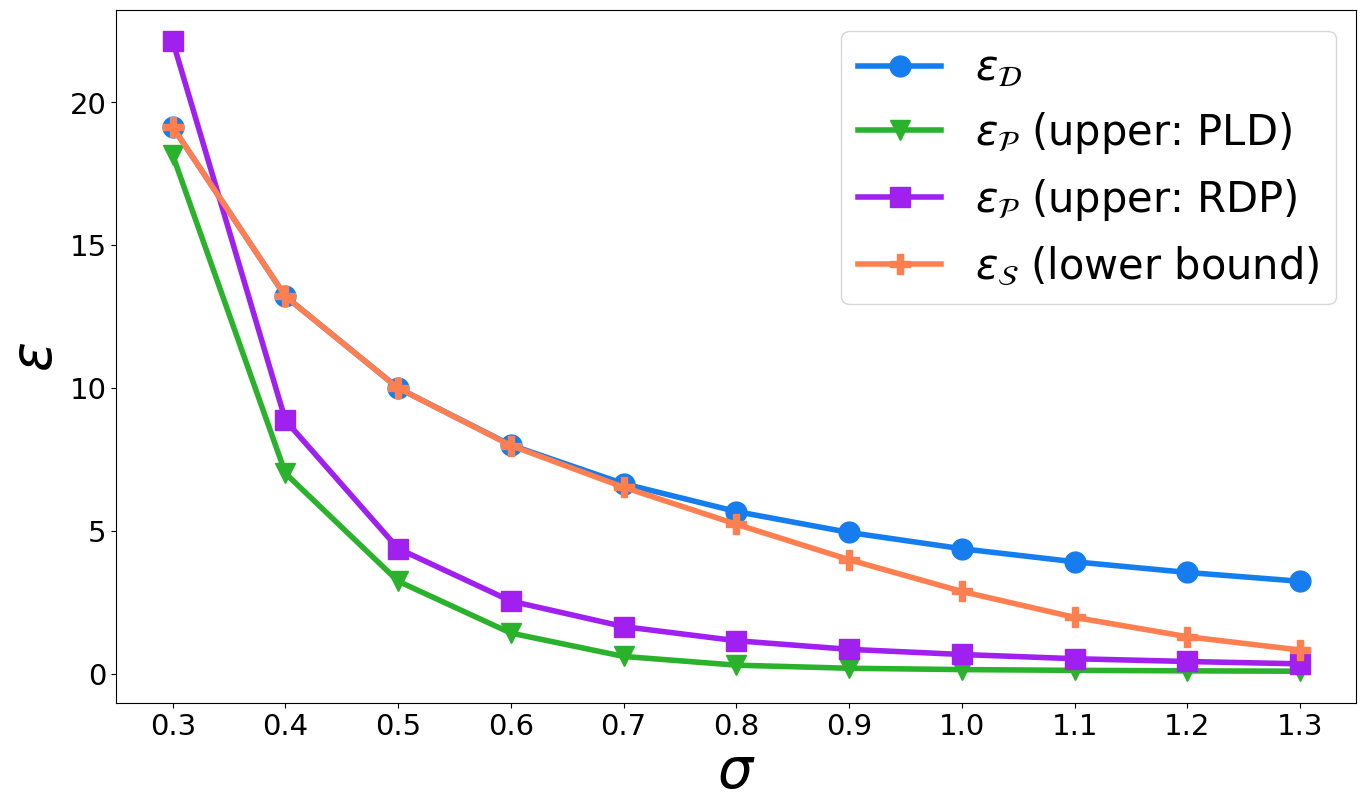}
\caption{$\epsD(\delta)$, upper bounds on $\epsP(\delta)$ and a lower bound on $\epsS(\delta)$ for varying $\sigma$ and fixed $\delta = 10^{-5}$ and $T = 1000$.}
\label{fig:eps-vs-sigma-3}
\end{figure}

\subsection{\boldmath $\delta$ vs. $\eps$ for fixed $\sigma$ and $T$}\label{subsec:delta_vs_eps}

Next, we plot $\delta$ against $\eps$ for fixed $\sigma$ and $T$. We compute upper bounds on $\deltaP(\eps)$ using R\'enyi DP (RDP) as well as using privacy loss distributions (PLD).

In particular when $\sigma$ is closer to $1.0$, we find that our lower bound on $\deltaS(\eps)$ is distinctly smaller than $\deltaD(\eps)$, but still significantly larger than $\deltaP(\eps)$.
\begin{itemize}[nosep]
\item In \Cref{fig:delta-vs-eps-2}, we set $\sigma=0.8$ and number of steps $T = 1000$. In particular, while $\deltaP(1) \le 9.873 \cdot 10^{-9}$ (as per PLD accounting) and $\deltaP(1) \le 3.346\cdot 10^{-5}$ (as per RDP accounting), we have that $\deltaS(1) \ge 0.018$ and $\deltaS(4) \ge 1.6 \cdot 10^{-4}$.

For $\eps = 4$, we find the upper bound using PLD accounting to be larger than the upper bound using R\'enyi DP accounting. This is attributable to the numerical instability in PLD accounting when $\delta$ is very small.

\begin{figure}[h]
\centering
\includegraphics[width=10cm]{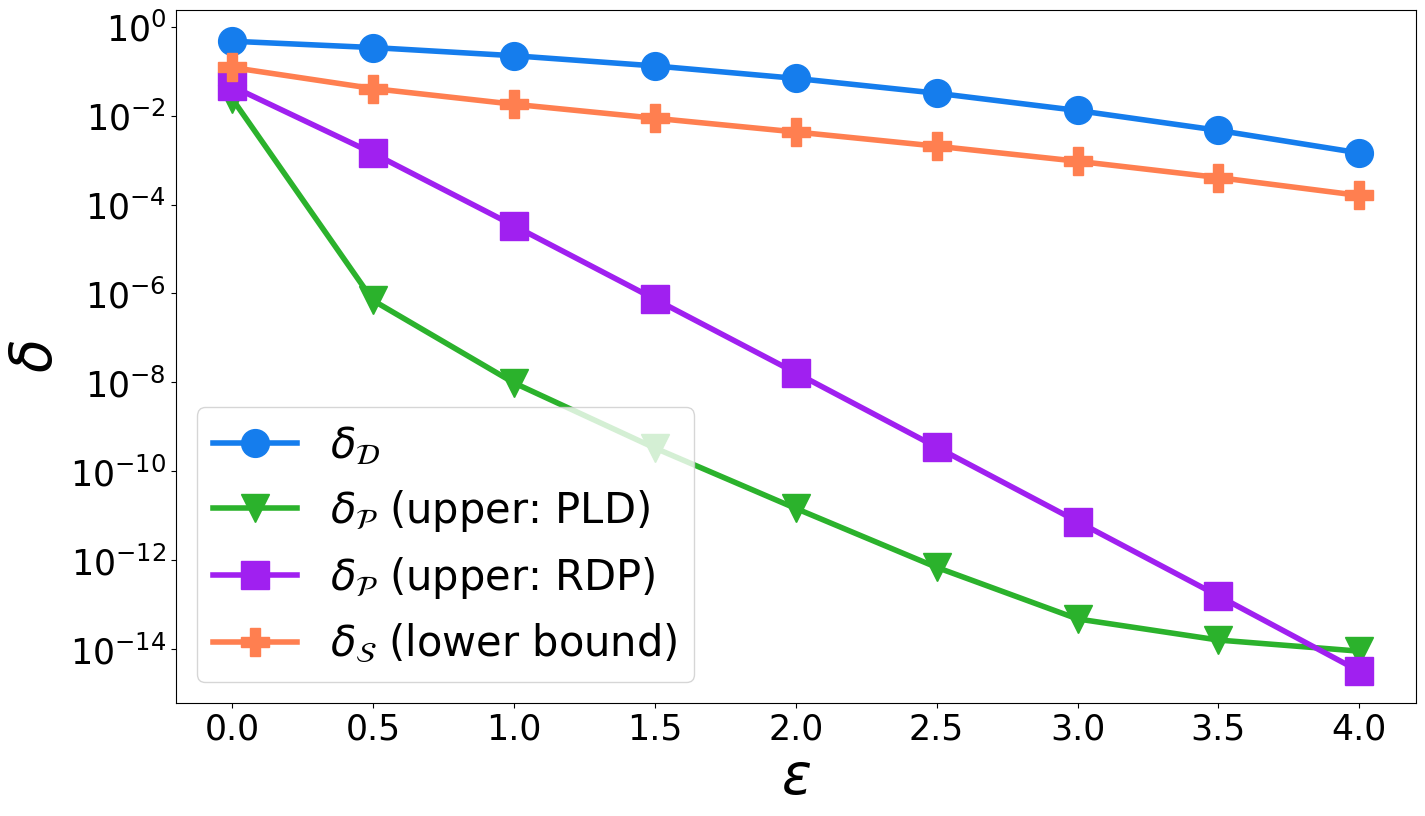}
\caption{$\deltaD(\eps)$, upper bounds on $\deltaP(\eps)$ and a lower bound on $\deltaS(\eps)$ for varying $\eps$ and fixed $\sigma = 0.8$ and $T = 1000$.}
\label{fig:delta-vs-eps-2}
\end{figure}

\item In \Cref{fig:delta-vs-eps-3}, we set $\sigma=1.0$ and number of steps $T = 1000$. In particular, while $\deltaP(1) \le 2.06 \cdot 10^{-10}$ (as per PLD accounting) and $\deltaP(1) \le 8.45 \cdot 10^{-5}$ (as per RDP accounting), we have that $\deltaS(1) \ge 0.004$ and $\deltaS(4) \ge 4.38 \cdot 10^{-7}$ (last one not shown in plot).

For $\eps > 1.0$, we find the upper bound using PLD accounting appears to not decrease as much, which could be due to the numerical instability in PLD accounting when $\delta$ is very small.

\begin{figure}[h]
\centering
\includegraphics[width=10cm]{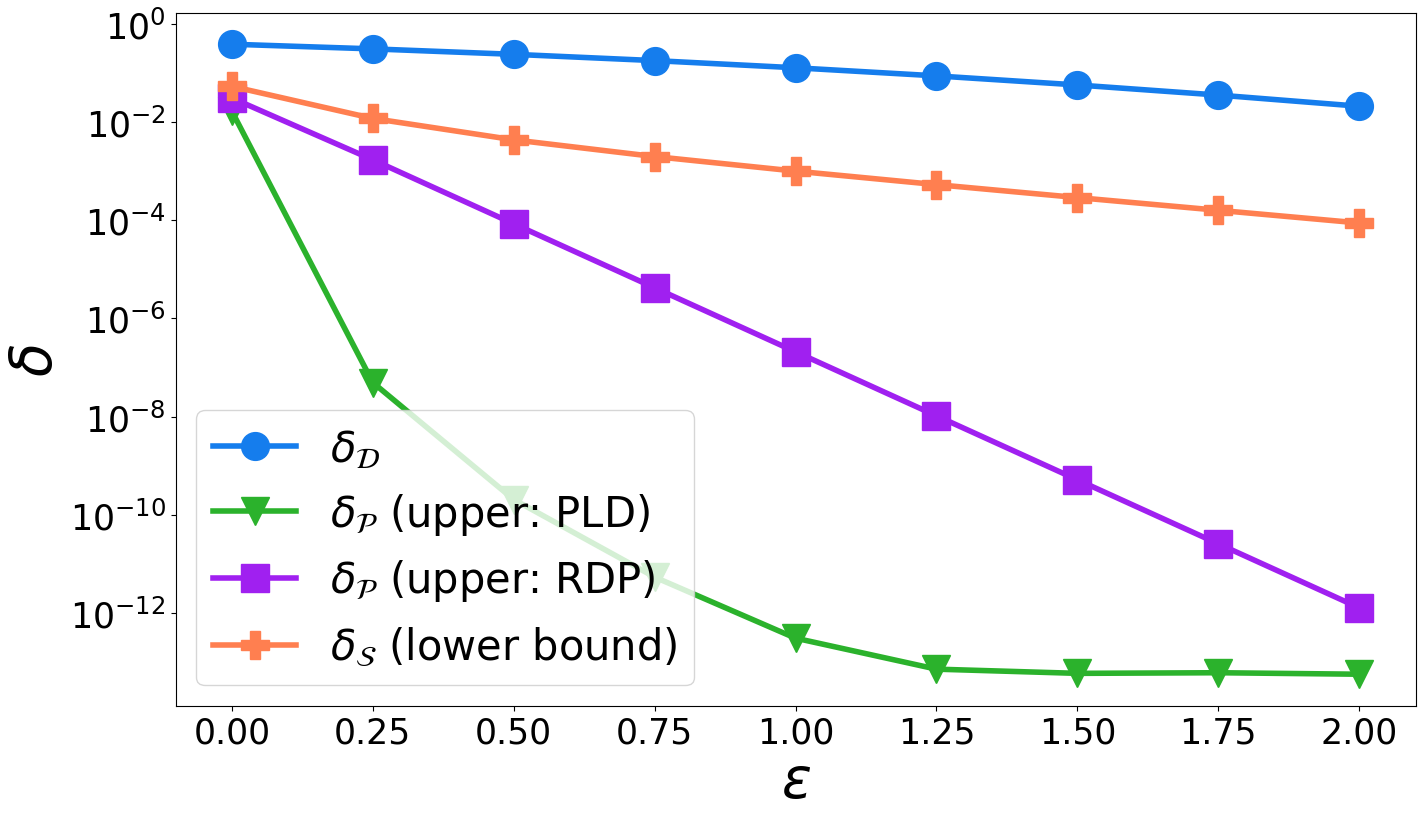}
\caption{$\deltaD(\eps)$, upper bounds on $\deltaP(\eps)$ and a lower bound on $\deltaS(\eps)$ for varying $\eps$ and fixed $\sigma = 1.0$ and $T = 1000$.}
\label{fig:delta-vs-eps-3}
\end{figure}

\end{itemize}

\section{\boldmath Proof of \texorpdfstring{\Cref{thm:D-vs-S}}{Theorem~\ref{thm:D-vs-S}}: ($\ABLQD$ vs $\ABLQS$)}\label{apx:D-vs-S}

We use the joint convexity property of hockey stick divergence. While this is standard, we include a proof for completeness.

\begin{lemma}[Joint Convexity of Hockey Stick Divergence]\label{lem:joint-convexity}
Given two mixture distributions $P = \sum_{i=1}^m \alpha_i P_i$ and $Q = \sum_{i=1}^m \alpha_i Q_i$, it holds for all $\eps \in \R$ that
$
D_{e^\eps}(P \| Q) ~\le~ \sum_i \alpha_i D_{e^\eps}(P_i \| Q_i)\,.
$
\end{lemma}
\begin{proof}
We have that
\begin{align*}
D_{e^{\eps}}(P \| Q)
&~=~ \sup_E \set{P(E) - e^{\eps} Q(E)}\\
&~=~ \sup_E \set{\sum_{i=1}^m \alpha_i \prn{P_i(E) - e^{\eps} Q_i(E)}} \\
&~\le~ \sum_{i=1}^m \alpha_i \cdot \sup_E \set{P_i(E) - e^\eps Q_i(E)}
~=~ \sum_{i=1}^m \alpha_i D_{e^\eps}(P_i \| Q_i)\,.\qedhere
\end{align*}  
\end{proof}

Thus, it follows that shuffling the dataset first cannot degrade the privacy guarantee of {\em any} mechanism as shown below.

\begin{lemma}\label{lem:some-amplification-by-shuffling}
Fix a mechanism $\cM : \cX^* \to \Delta_{\cO}$, and let $\cM_\cS$ be defined as $\cM_\cS(\bx) := \cM(\bx_{\pi})$ for a random permutation $\pi$ over $[n]$ where $\bx_{\pi} := (x_{\pi(1)}, \ldots, x_{\pi(n)})$. Then, if $\cM$ satisfies $(\eps, \delta)$-$\DP$, then $\cM_\cS$ also satisfies $(\eps, \delta)$-$\DP$.
\end{lemma}
\begin{proof}
Consider any adjacent pair of dataset $\bx \sim \bx'$. For any permutation $\pi$ over $[n]$, let
$P_{\pi} := \cM(\bx_{\pi})$ and $Q_{\pi} := \cM(\bx'_{\pi})$, and let $P = \cM_\cS(\bx)$ and $Q = \cM_\cS(\bx')$. It is easy to see that
\begin{align*}\textstyle
P ~=~ \sum_{\pi} \frac{1}{n!}\cdot P_{\pi}
\qquad \text{ and } \qquad
Q ~=~ \sum_{\pi} \frac{1}{n!}\cdot Q_{\pi}\,.
\end{align*}
Since $\cM$ satisfies $(\eps, \delta)$-$\DP$ it follows that $D_{e^\eps}(P_\pi \| Q_\pi) \le \delta$ for all permutations $\pi$.
Thus, from \Cref{lem:joint-convexity}, it follows that $D_{e^\eps}(P \| Q) \le \sum_\pi \frac{1}{n!} D_{e^\eps}(P_\pi \| Q_\pi) \le \delta$. Hence $\cM_{\cS}$ also satisfies $(\eps, \delta)$-$\DP$.
\end{proof}

\begin{proof}[Proof of \Cref{thm:D-vs-S}]
The proof follows by observing that if we choose $\cM = \ABLQD$ in \Cref{lem:some-amplification-by-shuffling}, then $\ABLQS$ is precisely the corresponding mechanism $\cM_\cS$.
\end{proof}

\section{\boldmath Proof of \texorpdfstring{\Cref{thm:D-vs-P-separation}}{Theorem~\ref{thm:D-vs-P-separation}} ($\ABLQD$ vs $\ABLQP$)}\label{apx:D-vs-P}

We first state and prove some intermediate statements required for the proof of \Cref{thm:D-vs-P-separation}. We use the Gaussian measure of a halfspace.
\begin{proposition}\label{prop:gaussian-halfspace-mass}
For $P = \cN(\mu, \sigma^2 I)$ and the set $E := \{ w \in \R^d : a^\top w - b \ge 0\}$, it holds that
$
P(E) = \Ncdf\prn{\frac{a^\top \mu - b}{\sigma \|a\|_2}}\,.
$
\end{proposition}

\begin{proposition}\label{prop:union-bound}
For all $T \in \N$ and distributions $A, B$, it holds that $D_1(A^{\otimes T} \| B^{\otimes T}) \leq 1 - (1 - D_1(A\|B))^T$.
And hence $D_1(A^{\otimes T} \| B^{\otimes T}) \leq T \cdot D_1(A\|B)$ and equality can occur only if $T = 1$ or $D_1(A\|B) = 0$.
\end{proposition}
\begin{proof}%
Recall that $D_1(A \| B)$ is the total variation distance between $A$ and $B$, which has the following characterization $\inf_{(X, Y)} \Pr[X \ne Y]$ where $(X, Y)$ is a coupling such that $X \sim A, Y \sim B$. Given any coupling $(X, Y)$ for $A, B$, we construct a coupling $((X_1, \dots, X_T), (Y_1, \dots, Y_T))$ of $A^{\otimes T}, B^{\otimes T}$ by sampling $(X_i, Y_i)$ independently according to the coupling $(X, Y)$. From this, we have
\begin{align*}
&\Pr[(X_1, \dots, X_T) \ne (Y_1, \dots, Y_T)] \\
&= 1 - \Pr[(X_1, \dots, X_T) = (Y_1, \dots, Y_T)] \\
&= 1 - \Pr[X=Y]^T.
\end{align*}
By taking the infimum over all $(X, Y)$ such that $X\sim A, Y \sim B$, we arrive at the desired bound.
\end{proof}

We also note that a simple observation that for all $P, Q$, the hockey stick divergence $D_{e^{\eps}}(P \| Q)$ is a $1$-Lipschitz in $e^\eps$.

\begin{proposition}\label{prop:hockey-lipschitz}
For $\eps_1 < \eps_2$, it holds that $D_{e^{\eps_1}}(P \| Q) - D_{e^{\eps_2}}(P \| Q) \le e^{\eps_2} - e^{\eps_1}$.
\end{proposition}
\begin{proof}
We have that 
\begin{align*}
D_{e^{\eps_1}}(P \| Q) - D_{e^{\eps_2}}(P \| Q)
&~=~ \sup_E \ [P(E) - e^{\eps_1} Q(E)] - \sup_{E'} \ [P(E') - e^{\eps_2} Q(E')]\\
&~\le~ \sup_E \ [P(E) - e^{\eps_1} Q(E) - P(E) + e^{\eps_2} Q(E)]\\
&~=~ \prn{e^{\eps_2} - e^{\eps_1}} \cdot \sup_E Q(E)\\
&~=~ e^{\eps_2} - e^{\eps_1}\,.
\qedhere
\end{align*}
\end{proof}

\begin{proof}[Proof of \Cref{thm:D-vs-P-separation}]
We prove each part as follows:

For part (a), first we consider the case of $\eps = 0$. In this case, $D_{e^{\eps}}(P \| Q)$ is simply the total variation distance between $P$ and $Q$. Recall that $\PP = A^{\otimes T}$ and $\QP = B^{\otimes T}$, where $A = (1-\frac1T) \QD + \frac1T \PD$ and $B = \QD$. Observe that $D_{1}(A\|B) = \frac1T \cdot D_{1}(\PD\|\QD)$. Thus, we have that
\begin{align*}
D_{1}(\PP \| \QP)
&=~ D_{1}(A^{\otimes T} \| B^{\otimes T})\\
&<~ T \cdot D_{1}(A \| B) \qquad \text{(\Cref{prop:union-bound})}\\
&=~ T \cdot \frac{1}{T} \cdot D_{1}(\PD \| \QD)\\
&=~ D_{1}(\PD \| \QD)\,.
\end{align*}
Note that the inequality is strict for $T > 1$.
Since $D_{e^{\eps}}(P \| Q)$ is continuous in $\eps$~(see \Cref{prop:hockey-lipschitz}),
there exists some $\eps_0 > 0$, such that for all $\eps \in [0, \eps_0)$, $\deltaD(\eps) > \deltaP(\eps)$.

For part (b), we construct an explicit bad event $E \subseteq \R^T$ such that $\PP(E) - e^{\eps} \QP(E) > D_{e^{\eps}}(\PD \| \QD)$. In particular, we consider:
\begin{align*}
    E := \textstyle \left\{ w \in \R^T \big | \sum_i w_i > (\eps + \log 2 + T \log T) \sigma^2 + {\frac{T}{2}}\right\}.
\end{align*}
The choice of $E$ is such that,
\begin{align*}
\textstyle\log \frac{\PP(w)}{\QP(w)}
&\textstyle~=~ \sum_{t=1}^T \log \frac{A(w_t)}{B(w_t)}\\
&\textstyle~=~ \sum_{t=1}^T \log \prn{1 - \frac1T + \frac1T \cdot e^{\frac{2w_t - 1}{2\sigma^2}}}\\
&\textstyle~\ge~ \sum_{t=1}^T \prn{\frac{2w_t - 1}{2\sigma^2} - \log T}\\
&\textstyle~\ge~ \frac{\sum_{t=1}^T w_t}{\sigma^2} - T \log T - \frac{T}{2\sigma^2}\\
&\textstyle~\ge~ \eps + \log 2 \qquad \text{ for all } w \in E\,.
\end{align*}
Hence it follows that $\log \frac{\PP(E)}{\QP(E)} \ge \eps + \log 2$, or equivalently, $\PP(E) \ge 2e^{\eps} \QP(E)$. This implies that $D_{e^{\eps}}(\PP \| \QP) \ge \frac12 \PP(E)$.

Next, we obtain a lower bound on $\PP(E)$. For $N_{\mu} := \cN(\mu, \sigma^2 I)$ and $p_k = \frac1{T^{-k}} (1 - \frac1T)^{T-k}$, it holds that
\begin{align}
\PP(E)
&\textstyle~=~ \sum_{\mu \in \{0, 1\}^T} p_{\|\mu\|_1} N_\mu(E)
~\ge~ N_{\mathbf{0}}(E) \nonumber\\
&\textstyle~=~ \Ncdf\prn{- \frac{\eps\sigma}{\sqrt{T}} - \frac{(T\log T + \log 2)\sigma}{\sqrt{T}} - \frac{\sqrt{T}}{2\sigma}}, \label{eq:deltaP-lower-bound}
\end{align}
where we use \Cref{prop:gaussian-halfspace-mass} in the last two steps.
On the other hand, we have from \Cref{prop:D-hockey} that
\begin{align}
\deltaD(\eps)
&\textstyle~=~ \Ncdf\prn{- \eps \sigma + \frac{1}{2\sigma}} - e^{\eps} \Ncdf\prn{- \eps \sigma - \frac{1}{2\sigma}}\nonumber\\
&\textstyle~\le~ \Ncdf\prn{- \eps \sigma + \frac{1}{2\sigma}}.\label{eq:deltaD-upper-bound}
\end{align}

There exists a sufficiently large $\eps_1$ such that
\begin{align}
\deltaP(\eps) &~\ge~ D_{e^{\eps}}(\PP \| \QP) \nonumber\\
&~\ge~ \textstyle \frac12 \PP(E)\nonumber\\
&~\ge~ \textstyle \frac12 \Ncdf\prn{- \frac{\eps\sigma}{\sqrt{T}} - \frac{(T\log T + \log 2)\sigma}{\sqrt{T}} - \frac{\sqrt{T}}{2\sigma}} \label{eq:eps-sigma-over-sqrt-T}\\
&~\ge~ \textstyle \Ncdf\prn{- \eps \sigma + \frac{1}{2\sigma}} \qquad \text{(for $\eps > \eps_1$)} \label{eq:eps-sigma}\\
&~\ge~ \deltaD(\eps)\,,\nonumber
\end{align}
by noting that for large $\eps$ the most significant term inside $\Ncdf(\cdot)$ in \eqref{eq:eps-sigma-over-sqrt-T} is $-\eps \sigma / \sqrt{T}$, whereas in \eqref{eq:eps-sigma} the most significant term inside $\Ncdf(\cdot)$ is $-\eps \sigma$, which decreases much faster as $\eps \to \infty$, for a fixed $T > 1$ and $\sigma > 0$.
\end{proof}
\end{document}